\documentclass[letterpaper]{article} 
\usepackage{aaai20}  
\usepackage{times}  
\usepackage{helvet} 
\usepackage{courier}  
\usepackage[hyphens]{url}  
\usepackage{graphicx} 
\usepackage{graphicx}  
\usepackage{amsthm}
\usepackage{amsmath}
\usepackage{listings}
\usepackage{algorithm}
\usepackage{pgfplots}
\usepackage{subcaption}
\usepackage{microtype}
\theoremstyle{definition}
\newtheorem{definition}{Definition}

\newtheorem{theorem}{Theorem}

\newtheorem{proposition}{Proposition}

\newcommand{\name}{Forgetgol}
\newcommand{\syn}{\name{}$_{\text{syn}}$}
\newcommand{\stat}{\name{}$_{\text{stat}}$}

\lstnewenvironment{myalgorithm}[1][] 
{
    \lstset{ 
        mathescape=true,
        numbers=left,
        escapeinside={*}{*},
        columns=flexible,
        numberstyle=\footnotesize,
        basicstyle=\footnotesize,
        keywordstyle=\footnotesize\bfseries,
        keywords={,input, output, and, return, datatype, not in, function, func, in, if, else, for, foreach, while, begin, end, }
        numbers=left,
        #1 
    }
}
{}
\nocopyright
\title{Forgetting to Learn Logic Programs}
\author{
 Andrew Cropper\\
 University of Oxford\\
 andrew.cropper@cs.ox.ac.uk
}
 
\begin{document}

\maketitle
\begin{abstract}
Most program induction approaches require predefined, often hand-engineered, background knowledge (BK).
To overcome this limitation, we explore methods to automatically acquire BK through multi-task learning.
In this approach, a learner adds learned programs to its BK so that they can be reused to help learn other programs.
To improve learning performance, we explore the idea of \emph{forgetting}, where a learner can additionally remove programs from its BK.
We consider forgetting in an inductive logic programming (ILP) setting.
We show that forgetting can significantly reduce both the size of the hypothesis space and the sample complexity of an ILP learner.
We introduce \emph{\name{}}, a multi-task ILP learner which supports forgetting.
We experimentally compare \name{} against approaches that either remember or forget everything.
Our experimental results show that \name{} outperforms the alternative approaches when learning from over 10,000 tasks.
\end{abstract}
\section{Introduction}
\label{sec:intro}

A major challenge in machine learning is inductive bias: how to choose a learner's hypothesis space so that it is large enough to contain the target hypothesis, yet small enough to be searched efficiently \cite{baxter}.
A major challenge in program induction is how to chose a learner's background knowledge (BK), which in turn defines the hypothesis space.
Most approaches need predefined, often hand-engineered, BK \cite{law:ilasp,deepcoder,evans:dilp,ellis:images}.
By contrast, we want a learner to automatically learn BK.

Several approaches \cite{dechter:ec,ellis:scc,crop:playgol} learn BK through \emph{meta-learning}.
The idea is to use knowledge gained from solving one problem to help solve a different problem.
For instance, Lin et al. (\citeyear{mugg:metabias}) use this approach on 17 string transformation problems.
They automatically identify easier problems, learn programs for them, and then reuse the learned programs to help learn programs for more difficult problems.
The authors experimentally show that their multi-task approach performs substantially better than a single-task approach because learned programs are frequently reused.
They also show that this approach leads to a hierarchical library of reusable programs.

However, no approach has been shown to work on many ($>$10,000) tasks.
Moreover, most approaches struggle even on few tasks \cite{foil,ferri:2001}.
By contrast, humans easily learn thousands of diverse concepts.
If we want to develop human-like AI systems that support lifelong learning -- a grand challenge in AI \cite{lifelong,lake:buildingmachines} -- then we need to overcome this limitation.

We think the key limitation with existing approaches is that they save all learned programs to the BK, which is a problem because too much irrelevant BK is detrimental to learning performance \cite{srinivasan:1995,srinivasan:2003}.
The Blumer bound \cite{blumer:bound}\footnote{The bound is a reformulation of Lemma 2.1.} helps explain this problem.
This bound states that given two hypothesis spaces, searching the smaller space will result in fewer errors compared to the larger space, assuming that the target hypothesis is in both spaces.
Here lies the problem: how to choose a learner's hypothesis space so that it is large enough to contain the target hypothesis yet small enough to be efficiently searched.

To address this problem, we explore the idea of \emph{forgetting}.
In this approach, a learner continually grows and shrinks its hypothesis space by adding and removing programs to and from its BK.
We claim that forgetting can improve learning performance, especially when learning from many tasks.
To support this claim, we make the following contributions:

\begin{itemize}
\item
We define forgetting in an inductive logic programming (ILP) setting (Section \ref{sec:framework}).
We show that forgetting can reduce both the size of the hypothesis space (Theorem \ref{thm:forgetspace}) and the sample complexity of an ILP learner (Theorem \ref{thm:forget:improvement}).
We show that optimal forgetting is NP-Hard (Theorem \ref{thm:nphard}).
\item
We introduce \emph{\name{}}, a multi-task ILP learner that continually learns, saves, and forgets programs (Section \ref{sec:impl}).
We introduce two forgetting methods: \emph{syntactical} and \emph{statistical} forgetting, the latter based on Theorem \ref{thm:forgetspace}.
\item
We experimentally show on two datasets (robots and Lego) that forgetting can substantially improve learning performance when learning from over 10,000 tasks (Section \ref{sec:experiments}).
This experiment is the first to consider the performance of a learner on many tasks.
\end{itemize}

\section{Related Work}
\label{sec:related}

Several approaches \cite{curled,dumancic:encoding,crop:playgol} learn BK though unsupervised learning.
These approaches do not require user-supplied tasks as input.
By contrast, we learn BK through supervised learning and need a corpus of tasks as input.
Our forgetting idea is, however, equally applicable to unsupervised approaches.

Supervised approaches perform either \emph{multi-task} \cite{caruana:mtl} or \emph{lifelong} \cite{lifelong} learning.
A multi-task learner is given a \emph{set} of tasks and can solve them in any order.
A lifelong learner is given a \emph{sequence} of tasks and can use knowledge gained from the past n-1 tasks to help solve the nth task.
The idea behind both approaches is to reuse knowledge gained from solving one problem to help to solve a different problem, i.e. perform \emph{transfer learning} \cite{torrey2010transfer}.
Although we focus on multi-task learning, our forgetting idea is suitable for both approaches.

Lin et al. (\citeyear{mugg:metabias}) state that an important problem with multi-task learning is how to handle the complexity of growing BK.
To address this problem, we propose forgetting, where a learner continually revises its BK by adding and removing programs, i.e. performs automatic bias-revision \cite{structenyears}.
Forgetting has long been recognised as an important feature in AI \cite{lin1994forget}.
Most forgetting work focuses on eliminating logical redundancy \cite{plotkin:thesis}, e.g. to improve efficiency in SAT \cite{heule2015clause}.
Our approach is slightly unusual because we are willing to forget logically irredundant knowledge.


Our forgetting idea comes from the observation that existing approaches remember everything \cite{foil,ferri:2001,mugg:metabias,dechter:ec,crop:playgol}, which is a problem because too much irrelevant BK is detrimental to learning performance \cite{srinivasan:1995,srinivasan:2003}.

A notable case of forgetting in machine learning is \emph{catastrophic forgetting} \cite{catastrophicforgetting}, the tendency of neural networks to forget previously learned knowledge upon learning new knowledge.
In multi-task program induction we have the inverse problem: \emph{catastrophic remembering}, the inability for a learner to forget knowledge.
We therefore need \emph{intentional forgetting} \cite{DBLP:journals/ki/BeierleT19}.

Forgetting in program induction mostly focuses on forgetting examples \cite{DBLP:conf/ijcai/SablonR95,DBLP:journals/ml/WidmerK96,Maloof:1997:PPM:266658}.
By contrast we forget BK.
Mart{\'{\i}}nez-Plumed et al. (\citeyear{fernando}) explore ways to compress BK without losing knowledge.
Our work differs in many ways, mainly because we work in a multi-task setting with many tasks.
Forgetting is essentially identifying irrelevant BK.
Deepcoder \cite{deepcoder} and Dreamcoder \cite{ellis:scc} both train neural networks to score how relevant programs are in the BK.
Whereas Deepcoder's BK is fixed and predetermined by the user, we grow and revise our BK through multi-task learning.
Dreamcoder also revises its BK through multi-task learning.
Our work differs from Dreamcoder because we (1) formally show that forgetting can improve learning performance, (2) describe forgetting methods that do not require neural networks, and (3) learn from 20,000 tasks, whereas Dreamcoder considers at most 236 tasks.

\section{Problem Setting}
\label{sec:framework}

We now define the forgetting problem, show that optimal forgetting is NP-hard, and that forgetting can reduce both the size of the hypothesis space and the sample complexity of a learner.

\subsection{ILP Problem}
We base our problem on the ILP learning from entailment setting \cite{luc:book}.
We assume standard logic programming definitions throughout \cite{lloyd:book}.
We define the input:

\begin{definition}[\textbf{ILP input}]
An ILP input is a tuple $(B,E^+,E^-)$ where $B$ is logic program background knowledge, and $E^+$ and $E^-$ are sets of ground atoms which represent positive and negative examples respectively.
\end{definition}

\noindent
By \emph{logic program} we mean a set of definite clauses, and thus any subsequent reference to a clause refers to a definite clause.
We assume that $B$ contains the necessary language bias, such as mode declarations \cite{mugg:progol} or metarules \cite{crop:reduce}, to define the hypothesis space:

\begin{definition}[\textbf{Hypothesis space}]
The hypothesis space $\mathcal{H}_{B}$ is the set of all hypotheses (logic programs) defined by background knowledge $B$.
\end{definition}

\noindent
The version space contains all complete (covers all the positive examples) and consistent (covers none of the negative examples) hypotheses:

\begin{definition}[\textbf{Version space}]
Let $(B,E^+,E^-)$ be an ILP input.
Then the version space is $\mathcal{V}_{B,E^+,E^-} = \{H  \in \mathcal{H}_B | (H \cup B \models E^+) \land  (\forall e^- \in E^-, H \cup B \not\models e^-) \}$.

\end{definition}

\noindent
We define an optimal hypothesis:

\begin{definition}[\textbf{Optimal hypothesis}]
Let $I = (B,E^+,E^-)$ be an ILP input and $cost: \mathcal{H} \rightarrow R$ be an arbitrary cost function.
Then $H \in \mathcal{V}_{B,E^+,E^-}$ is an \emph{optimal hypothesis} for $I$ if and only if $cost(H) \leq cost(H')$ for all $H' \in \mathcal{V}_{B,E^+,E^-}$.
\end{definition}

\noindent
The cost of a hypothesis could be many things, such as the number of literals in it \cite{law:ilasp} or its computational complexity \cite{crop:metaopt}.
In this paper, the cost of a hypothesis is the number of clauses in it.



The ILP problem is to find a complete and consistent hypothesis:

\begin{definition}[\textbf{ILP problem}]
Given an ILP input $(B,E^+,E^-)$, the ILP problem is to return a hypothesis $H \in \mathcal{V}_{B,E^+,E^-}$.
\end{definition}

\noindent
A learner should ideally return the optimal hypothesis, but it is not an absolute requirement, and it is common to sacrifice optimality for efficiency.

\subsection{Forgetting}

The forgetting problem is to find a subset of the given BK from which one can still learn the optimal hypothesis:

\begin{definition}[\textbf{Forgetting problem}]
\label{def:forgetprob}
Let $I = (B,E^+,E^-)$ be an ILP input and $H$ be an optimal hypothesis for $I$.
Then the \emph{forgetting problem} is to return $B' \subset B$ such that $H \in \mathcal{V}_{B',E^+,E^-}$.
We say that $B'$ is \emph{reduced} background knowledge for $E^+$ and $E^-$.
\end{definition}

\noindent
We would ideally perform optimal forgetting:

\begin{definition}[\textbf{Optimal forgetting}]
Let $I = (B,E^+,E^-)$ be an ILP input and $H$ be an optimal hypothesis for $I$.
Then the \emph{optimal forgetting} problem is to return $B' \subset B$ such that $H \in \mathcal{V}_{B',E^+,E^-}$ and there is no $B'' \subset B'$ such that $H \in \mathcal{V}_{B'',E^+,E^-}$.
\end{definition}






\noindent
Even if we assume that the forgetting problem is decidable, optimal forgetting is NP-hard:

\begin{theorem}
[\textbf{Complexity}]
\label{thm:nphard}
Optimal forgetting is NP-Hard.
\end{theorem}
\begin{proof}
We can reduce the knapsack problem, which is NP-Hard, to optimal forgetting.
\end{proof}

\subsection{Meta-Interpretive Learning}
We now show the benefits of forgetting in meta-interpretive (MIL) \cite{mugg:metagold,crop:metaho}, a powerful form of ILP that supports predicate invention \cite{stahl:pi} and learning recursive programs.
For brevity, we omit a formal description of MIL and instead provide a brief overview.
A MIL learner is given as input two sets of ground atoms that represent positive and negative examples of a target concept, BK described as a logic program, and a set of higher-order Horn clauses called metarules (Table \ref{tab:metarules}).
A MIL learner uses the metarules to construct a proof of the positive examples and none of the negative examples, and forms a hypothesis using the substitutions for the variables in the metarules.
For instance, given positive and negative examples of the \emph{grandparent} relation, background knowledge with the \emph{parent} relation, and the metarules in Table \ref{tab:metarules}, a MIL learner could use the \emph{chain} metarule with the substitutions \{P/grandparent, Q/parent, R/parent\} to induce the theory:
\emph{grandparent(A,B) $\leftarrow$ parent(A,C), parent(C,B)}.


\begin{table}[ht]
\centering
\begin{tabular}{|l|l|}
\hline
\textbf{Name} & \textbf{Metarule}\\ \hline
ident & $P(A,B) \leftarrow Q(A,B)$\\
precon & $P(A,B) \leftarrow Q(A),R(A,B)$\\
postcon & $P(A,B) \leftarrow Q(A,B),R(B)$\\
chain & $P(A,B) \leftarrow Q(A,C),R(C,B)$\\
\hline
\end{tabular}
\caption{
Example metarules.
The letters $P$, $Q$, and $R$ denote second-order variables.
The letters $A$, $B$, and $C$ denote first-order variables.
}
\label{tab:metarules}
\end{table}

\noindent
The size of the MIL hypothesis space is a function of the metarules, the number of predicate symbols in the BK, and a target hypothesis size.
We define $sig(P)$ to be the set of predicate symbols in the logic program $P$.
In the next two propositions, assume an ILP input with background knowledge $B$, where $p = |sig(B)|$, $m$ is the number of metarules in $B$, each metarule in $B$ has at most $j$ body literals, and $n$ is a target hypothesis size.
Cropper et al. (\citeyear{crop:metaho}) show that the size of the MIL hypothesis space is:

\begin{proposition}[\textbf{Hypothesis space}]
\label{prop:hspace}
The size of the MIL hypothesis space is at most $(mp^{j+1})^n$.
\end{proposition}

\noindent
The authors use this result with the Blumer bound to show the MIL sample complexity in a PAC learning setting \cite{paclearning}:

\begin{proposition}[\textbf{Sample complexity}]
\label{prop:sampcomp}
With error $\epsilon$ and confidence $\delta$ MIL has sample complexity
$s \geq \frac{1}{\epsilon} (n \ln(m) + (j+1)n \ln(p) + \ln\frac{1}{\delta})$.
\end{proposition}

\noindent
We use these results in the next section to show that forgetting can reduce sample complexity in MIL and in Section \ref{sec:statistical} to define a statistical forgetting method.

\subsection{Forgetting Sample Complexity}

The purpose of forgetting is to reduce the size of the hypothesis space without excluding the target hypothesis.
In MIL we want to reduce the number of predicate symbols without excluding the target hypothesis from the hypothesis space.
We now show the potential benefits of doing so.
In the next two theorems, assume an ILP input $(B,E^+,E^-)$ and let $B' \subset B$ be reduced background knowledge for $E^+$ and $E^-$, where $B$ and $B'$ both contain $m$ metarules with at most $j$ body literals, $p = |sig(B)|$, $p' = |sig(B')|$, $n$ be a target hypothesis size, and $r = p'/p$, i.e. $r$ is the forgetting \emph{reduction ratio}.

\begin{theorem}[\textbf{Hypothesis space reduction}]
\label{thm:forgetspace}
Forgetting reduces the size of the MIL hypothesis space by a factor of $r^{(j+1)n}$.
\end{theorem}
\begin{proof}
Replacing $p$ with $rp$ in Proposition \ref{prop:hspace} and rearranging terms leads to $r^{(j+1)n} (mp^{(j+1)})^n$.
\end{proof}

\begin{theorem}[\textbf{Sample complexity reduction}]
\label{thm:forget:improvement}
Forgetting reduces sample complexity in MIL by $(j+1)n \ln(r)$.
\end{theorem}
\begin{proof}
Replacing $p$ with $rp$ in Proposition \ref{prop:sampcomp} and rearranging terms leads to
$s \geq \frac{1}{\epsilon} (n \ln(m) + (j+1)n \ln(r) + (j+1)n \ln(p) + \ln\frac{1}{\delta})$.
\end{proof}

\noindent
These results show that forgetting can substantially reduce the size of the hypothesis space and sample complexity of a learner.

\section{\name{}}
\label{sec:impl}

We now introduce \name{}, a multi-task ILP system which supports forgetting.
\name{} relies on Metagol \cite{metagol}, a MIL system.
We first briefly describe Metagol.

\subsection{Metagol}
Metagol is based on a Prolog meta-interpreter.
Metagol takes as input (1) a logic program which represents BK, (2) a set of predicate symbols $S$ which denotes symbols that may appear in a hypothesis, similar to determinations in Aleph \cite{aleph}, (3) two sets of ground atoms which represent positive and negative examples, and (4) a maximum program size.
The set $S$ is necessary because the BK may contain relations which should not appear in a hypothesis.
For instance, the BK could contain a definition for \emph{quicksort} which uses \emph{partition} and \emph{append} but we may want Metagol to only use \emph{quicksort} in a hypothesis.
The set $S$ is usually defined as part of the BK but we separate it for clarity.

Given these inputs, Metagol tries learn a logic program by proving each atom in the set of positive examples.
Metagol first tries to prove an atom using the BK by delegating the proof to Prolog.
If this step fails, Metagol tries to unify the atom with the head of a metarule and tries to bind the variables in a metarule to symbols in $S$.
Metagol saves the substitutions and tries to prove the body of the metarule recursively through meta-interpretation.
After proving all atoms, Metagol induces a logic program by projecting the substitutions onto the corresponding metarules.
Metagol checks the consistency of the induced program with the negative examples.
If the program is inconsistent, Metagol backtracks to consider alternative programs.
Metagol uses iterative deepening to ensure that the learned program has the minimal number of clauses.
Metagol supports predicate invention by adding $d-1$ new predicate symbols to $S$ at each depth $d$.


\subsection{\name{}}
\name{} (Algorithm \ref{alg1}) takes as input (1) a logic program which represents BK, (2) a set of predicate symbols $S$ which denotes those that may appear in a hypothesis, (3) a set of tasks $T$ (a set of pairs of sets of positive and negative examples), (4) a maximum program size, and (5) a forgetting function. 
\name{} uses dependent learning \cite{mugg:metabias} to learn programs and to expand $B$ and $S$.
Starting at depth $d=1$, \name{} tries to learn a program for each task using at most $d$ clauses.
To learn a program, \name{} calls Metagol (line 17).
Once \name{} has tried to learn a program for each task at depth $d$, it updates $B$ (line 6) with the learned programs, increases the depth, and tries to learn programs for the remaining tasks (line 8).
\name{} also updates $S$ with learned predicate symbols (line 7), including invented symbols.
\name{} repeats this process until it reaches a maximum depth, at which point it returns the learned programs.
In contrast to the approach used by Lin et al, \name{} can also remove (forget) elements from $S$ (line 4).



\begin{algorithm}[t]
\begin{myalgorithm}[]
func $\text{forgetgol}$(B,S,T,max$_d$,forget):
  P = {}
  for depth=1 to max$_d$:
    S$'$ = forget(S,B)
    S$_d$, T$_d$, P$_d$, = learn(B,S$'$,T,depth)
    B = B $\cup$ P$_d$
    S = S $\cup$ S$_d$
    T = T $\setminus$ T$_d$
    P = P $\cup$ P$_d$
  return P

func learn(B,S,T,depth):
  S$_d$ = {}
  T$_d$ = {}
  P$_d$ = {}
  for (E$^+$,E$^-$) in T:
    prog = metagol(B,S,E$^+$,E$^-$,depth)
    if prog != null:
      S$_d$ = S$_d$ $\cup$ {p | p is the head symbol of a clause $\text{in}$ prog}
      T$_d$ = T$_d$ $\cup$ {(E$^+$,E$^-$)}
      P$_d$ = P$_d$ $\cup$ prog
  return S$_d$, T$_d$, P$_d$
\end{myalgorithm}
\caption{\name{}}
\label{alg1}
\end{algorithm}

\subsection{Forgetting}
The purpose of forgetting is to reduce the size of the hypothesis space without excluding the target hypothesis.
\name{} reduces the size of the hypothesis space by reducing the number of predicate symbols allowed in a hypothesis (the set $S$ in Algorithm \ref{alg1}).
To be clear: \name{} does not remove clauses from the BK.
\name{} removes predicate symbols from $S$ because, as stated, the number of predicate symbols determines the size of the hypothesis space (Proposition \ref{prop:hspace}).
We consider two forgetting methods: \emph{syntactical} and \emph{statistical}.

\subsubsection{Syntactical Forgetting}
\label{sec:syntactical}







Algorithm \ref{alg:syntactical} shows our syntactical forgetting method, which removes syntactically duplicate programs.
We use the Tamaki and Sato unfold operation \cite{unfolding} on each clause in the BK to replace invented predicates with their definitions (but not when the predicate refers to the head of the clause, i.e. when it is used recursively).
For each unfolded clause, we check whether (1) the head symbol of the clause is in S (i.e. is allowed in a hypothesis), and (2) we have already seen a clause with the same head arguments and the same body (line 8).
If so, we forget the head predicate symbol; otherwise, we keep the symbol and add the unfolded clause to the seen set.

\begin{algorithm}[t]
\begin{myalgorithm}[]
func forget(S,B):
  S$'$ = {}
  B$'$ = {}
  for clause in B:
    s = head_sym(clause)
    if s in S:
      clause$'$ = unfold(clause,B)
      if clause$'$ is new with respect to B$'$:
        S$'$ = S$'$ $\cup$ {s}
        B$'$ = B$'$ $\cup$ {clause$'$}
  return S$'$
\end{myalgorithm}
\caption{Syntactical forgetting}
\label{alg:syntactical}
\end{algorithm}

\subsubsection{Statistical Forgetting}
\label{sec:statistical}
Statistical forgetting is based on the hypothesis space results in Section \ref{sec:framework}.
Deciding whether to keep or forget a predicate symbol (an element of $S$) in MIL depends on (1) the cost of relearning it, and (2) how likely it is to be reused.
As Proposition \ref{prop:hspace} states, given $p$ predicate symbols and $m$ metarules with at most $j$ body literals, the number of programs expressible with $n$ clauses is at most $(mp^{j+1})^n$.
Table \ref{tab:costs} uses this result to show the costs of saving or forgetting a predicate symbol when it is relevant or irrelevant.
However, we do not know beforehand whether a predicate symbol is relevant to future tasks.
We can, however, estimate relevancy in our multi-task setting using the relative reuse of a symbol, similar to Dechter et al. (\citeyear{dechter:ec}).
Specifically, we define the probability $Pr(s,B)$ that a predicate symbol $s$ is relevant given background knowledge $B$ as:

\[ Pr(s,B) = \frac{|\{ clause \in B | \text{\emph{ s in body of clause }}\}| + 1}{|B| + 1} \]

\begin{table}[ht]
    \centering
    \begin{tabular}{l|ll}
    & Relevant & Irrelevant\\
    \hline
    Keep & $(m(p+1)^{j+1})^{n-k}$ & $(m(p+1)^{j+1})^n$\\
    Forget & $(mp^{j+1})^n$ & $(mp^{j+1})^n$
    \end{tabular}
    \caption{
    The costs of keeping or forgetting a predicate symbol that is the head symbol of a definition composed of $k$ clauses.
    These costs are based on Proposition \ref{prop:hspace}, where $p$ is the number of predicate symbols, $m$ is the number of metarules with at most $j$ body literals, and $n$ is the number of clauses in the target program.
    The $n-k$ exponent when keeping a relevant symbol is because reusing a learned symbol reduces the size of the target program by $k$ clauses.
    }
    \label{tab:costs}
\end{table}

\noindent
The +1 values denote additive smoothing.
With this probability, we define the expected cost \emph{cost\_keep(s,B)} of keeping a symbol $s$ that is the head of $k$ clauses when searching for a program with $n$ clauses:

\begin{center}
\begin{tabular}{rl}
  \emph{cost\_keep(s,B) =} & $(pr(s,B)(m(p+1)^{j+1})^{n-k})$ \\
  & $ + \; ((1-pr(s,B)) (m(p+1)^{j+1})^n)$
\end{tabular}
\end{center}

\noindent
We can likewise define the expected cost \emph{cost\_forget(s,B)} = $(mp^{j+1})^n$ of forgetting $s$.
These costs allow us to define the forget function shown in Algorithm \ref{alg:statistical}.

\begin{algorithm}[t]
\begin{myalgorithm}[]
func forget(S,B):
  S$'$ = {}
  for clause in B:
    s = head_sym(clause)
    if s in S and cost_forget(s,B) > cost_keep(s,B):
      S$'$ = S$'$ $\cup$ {s}
  return S$'$
\end{myalgorithm}
\caption{Statistical forgetting}
\label{alg:statistical}
\end{algorithm}

\section{Experiments}
\label{sec:experiments}



To test our claim that forgetting can improve learning performance, our experiments aim to answer the question:
\begin{description}
\item[Q1] Can forgetting improve learning performance?
\end{description}

\noindent
To answer \textbf{Q1} we compare three variations of \name{}:

\begin{itemize}
    \item \textbf{\syn{}}: \name{} with syntactical forgetting
    \item \textbf{\stat{}}: \name{} with statistical forgetting
    \item \textbf{Metabias}: \name{} set to remember everything (we call Algorithm \ref{alg1} with a forget function that returns all the BK), which is equivalent to approach used in \cite{mugg:metabias}
\end{itemize}

\noindent
We introduced forgetting because we claim that learners which remember everything will become overwhelmed by too much BK when learning from many tasks.
To test this claim, our experiments aim to answer the question:

\begin{description}
\item[Q2] Do remember everything learners become overwhelmed by too much BK when learning from many tasks?
\end{description}

\noindent
To answer \textbf{Q2}, we compare the performance of the learners on progressively more tasks.
We expect that Metabias will improve given more tasks but will eventually become overwhelmed by too much BK, at which point its performance will start to degrade.
By contrast, we expect that \name{} will improve given more tasks and should outperform Metabias when given many tasks because it can forget BK.

We also compare \name{} and Metabias against Metagol.
However, this comparison cannot help us answer questions \textbf{Q1} and \textbf{Q2}, which is also why we do not compare \name{} against other program induction systems.
The purpose of this comparison is to add more experimental evidence to the results of Lin et al. (\citeyear{mugg:metabias}), where they compare multi-task learning with single-task learning.
We expect that Metagol will not improve given more tasks because it cannot reuse learned programs.

All the experimental data are available at https://github.com/andrewcropper/aaai20-forgetgol.

\subsection{Experiment 1 - Robot Planning}

Our first experiment is on learning robot plans.

\subsubsection{Materials}
A robot and a ball are in a $6 \times 6$ space.
A state describes the position of the robot, the ball, and whether the robot is holding the ball.
A training example is an atom $f(s_1,s_2)$, where $f$ is the target predicate and $s_1$ and $s_2$ are initial and final states respectively.
The task is to learn a program to move from the initial to the final state.
We generate training examples by generating random states, with the constraint that the robot can only hold the ball if it is in the same position as the ball.
The robot can perform the actions \emph{up}, \emph{down}, \emph{right}, \emph{left}, \emph{grab}, and \emph{drop}.
The learners use the metarules in Table \ref{tab:metarules}.

\subsubsection{Method}
Our dataset contains $n$ tasks for each $n$ in $\{2000,4000,\dots,20000\}$.
Each task is a one-shot learning task and is given a unique predicate symbol.
We enforce a timeout of 60 seconds per task per search depth.
We set the maximum program size to 6 clauses.
We measure the percentage of tasks solved (tasks where the learner learned a program) and learning times.
We plot the standard error of the means over 5 repetitions.

\subsubsection{Results}
\label{sec:robores}
Figure \ref{fig:robot-results-small} shows the results on the small dataset.
As expected, Metagol's performance does not vary when given more tasks.
By contrast, \name{} and Metabias continue to improve given more tasks, both in terms of percentage of tasks solved and learning time.
Figure \ref{fig:robot-results-big} shows the results on the big dataset.
We did not run Metagol on the big dataset because the results on the small dataset are sufficiently conclusive.
On the big dataset, the performances of Metabias and \stat{} start to degrade when given more than 8000 tasks.
However, this performance decrease is tiny ($<$1\%).
By contrast, \syn{} always solves 100\% of the tasks and by 20,000 tasks learns programs twice as quick as Metabias and \stat{}.

These results are somewhat unexpected.
We thought that Metabias would eventually become overwhelmed by too much BK, at which point its performance would start to degrade.
Metabias does not strongly exhibit this behaviour (the performance decrease is tiny) and performs well despite large amounts of BK.
After analysing the experimental results, Metabias rarely has to learn a program with more than two clauses because it could almost always reuse a learned program.
The performance of \stat{} is similar to Metabias because \stat{} rarely forgets anything, which suggests limitations with the proposed statistical forgetting method.

Overall, these results suggest that the answer to \textbf{Q1} is yes and the answer to \textbf{Q2} is no.

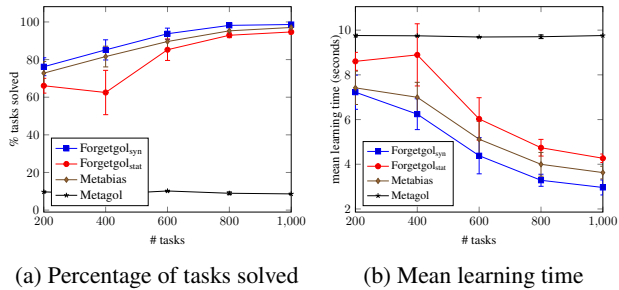
\begin{figure}[ht]
\centering
\begin{subfigure}{.5\linewidth}
\centering
\pgfplotsset{scaled x ticks=false}
\begin{tikzpicture}[scale=.48]
    \begin{axis}[
    xlabel=\# tasks,
    ylabel=\% tasks solved,
    xmin=200,xmax=1000,
    ylabel style={yshift=-4mm},
    legend style={legend pos=south west,style={nodes={right}}}
    ]

\addplot+[mark=square*,error bars/.cd,y fixed,y dir=both,y explicit] coordinates {
(200,76.2) +- (0,4.7502631506054485)
(400,85.15) +- (0,5.351752049562835)
(600,93.73333333333333) +- (0,2.975175064877133)
(800,98.175) +- (0,0.16583123951776998)
(1000,98.64000000000001) +- (0,0.3841874542459704)
};

\addplot+[mark=oplus*,error bars/.cd,y fixed,y dir=both,y explicit] coordinates {
(200,66.1) +- (0,3.9350984739902)
(400,62.5) +- (0,11.778688381988887)
(600,85.23333333333332) +- (0,5.694197826637998)
(800,92.925) +- (0,1.2652815892124567)
(1000,94.67999999999999) +- (0,1.1015443704181844)
};

\addplot+[mark=diamond*,error bars/.cd,y fixed,y dir=both,y explicit] coordinates {
(200,72.8) +- (0,6.224146527838173)
(400,81.6) +- (0,5.5003408985262)
(600,89.6) +- (0,3.785425265996469)
(800,95.275) +- (0,2.0222975794872524)
(1000,97.05999999999999) +- (0,0.7991245209602821)
};

\addplot+[error bars/.cd,y fixed,y dir=both,y explicit] coordinates {
(200,9.7) +- (0,0.40620192023179796)
(400,8.45) +- (0,0.3984344362627308)
(600,10.166666666666666) +- (0,0.3374742788552763)
(800,8.95) +- (0,0.7527034608662297)
(1000,8.620000000000001) +- (0,0.32155870381627055)
};

    \legend{\syn{},\stat{},Metabias,Metagol}
    \end{axis}
  \end{tikzpicture}
\caption{Percentage of tasks solved}
\end{subfigure}%
\begin{subfigure}{.5\linewidth}
\centering
\pgfplotsset{scaled x ticks=false}
\begin{tikzpicture}[scale=.48]
    \begin{axis}[
    xlabel=\# tasks,
    ylabel=mean learning time (seconds),
    xmin=200,xmax=1000,
    ylabel style={yshift=-7mm},
    legend style={legend pos=south west,font=\small,style={nodes={right}}}
    ]

\addplot+[mark=square*,error bars/.cd,y fixed,y dir=both,y explicit] coordinates {
(200,7.223253586769104) +- (0,0.7652005292661443)
(400,6.240336181759834) +- (0,0.6898511917844393)
(600,4.38307755335172) +- (0,0.8107558440890233)
(800,3.2833407340645793) +- (0,0.26928366401203985)
(1000,2.96140545539856) +- (0,0.33463864605737903)
};

\addplot+[mark=oplus*,error bars/.cd,y fixed,y dir=both,y explicit] coordinates {
(200,8.603900444507598) +- (0,0.40189760571376165)
(400,8.893593466520308) +- (0,1.3910501844874323)
(600,6.027590248187383) +- (0,0.9495290686162104)
(800,4.739445650935173) +- (0,0.37314974182698424)
(1000,4.268552606868743) +- (0,0.18974175860217038)
};

\addplot+[mark=diamond*,error bars/.cd,y fixed,y dir=both,y explicit] coordinates {
(200,7.418438065052032) +- (0,0.744844386629517)
(400,6.995693395376206) +- (0,0.6697459691988723)
(600,5.120112765550613) +- (0,0.8041587902576252)
(800,3.997431905031205) +- (0,0.5317494647742028)
(1000,3.6275448770999907) +- (0,0.28303090572800227)
};

\addplot+[error bars/.cd,y fixed,y dir=both,y explicit] coordinates {
(200,9.759010304927827) +- (0,0.05762496128492244)
(400,9.746609765529632) +- (0,0.019157029484439477)
(600,9.692933224042259) +- (0,0.01955825220712204)
(800,9.70545234155655) +- (0,0.08317288636390827)
(1000,9.757821498775481) +- (0,0.010339891057002988)
};

    \legend{\syn{},\stat{},Metabias,Metagol}
    \end{axis}
  \end{tikzpicture}
\caption{Mean learning time}
\end{subfigure}
\caption{Robot experiment small dataset results.}
\label{fig:robot-results-small}
\end{figure}

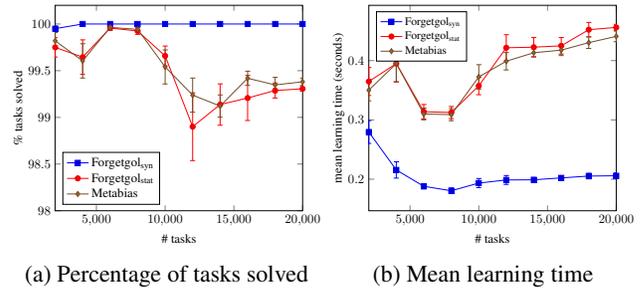
\begin{figure}[ht]
\centering
\begin{subfigure}{.5\linewidth}
\centering
\pgfplotsset{scaled x ticks=false}
\begin{tikzpicture}[scale=.48]
    \begin{axis}[
    xlabel=\# tasks,
    ylabel=\% tasks solved,
    xmin=2000,xmax=20000,
    ymin=98,
    ylabel style={yshift=-2mm},
    legend style={legend pos=south west,style={nodes={right}}}
    ]

\addplot+[mark=square*,error bars/.cd,y fixed,y dir=both,y explicit] coordinates {
(2000,99.95) +- (0,0.02738612787525675)
(4000,100.0) +- (0,0.0)
(6000,100.0) +- (0,0.0)
(8000,100.0) +- (0,0.0)
(10000,100.0) +- (0,0.0)
(12000,100.0) +- (0,0.0)
(14000,100.0) +- (0,0.0)
(16000,100.0) +- (0,0.0)
(18000,100.0) +- (0,0.0)
(20000,100.0) +- (0,0.0)
};

\addplot+[mark=oplus*,error bars/.cd,y fixed,y dir=both,y explicit] coordinates {
(2000,99.75) +- (0,0.10368220676663818)
(4000,99.64500000000001) +- (0,0.18513508581573554)
(6000,99.95666666666666) +- (0,0.023333333333332002)
(8000,99.92500000000001) +- (0,0.036443449342785074)
(10000,99.65799999999999) +- (0,0.10608487168300851)
(12000,98.9) +- (0,0.36271545321367216)
(14000,99.13714285714285) +- (0,0.22141013748153526)
(16000,99.20750000000001) +- (0,0.24119980047670075)
(18000,99.28777777777778) +- (0,0.08063023358456417)
(20000,99.306) +- (0,0.025612496949731178)

};

\addplot+[mark=diamond*,error bars/.cd,y fixed,y dir=both,y explicit] coordinates {
(2000,99.82) +- (0,0.08888194417315651)
(4000,99.60499999999999) +- (0,0.18377975949489056)
(6000,99.96333333333334) +- (0,0.032659863237110016)
(8000,99.94250000000001) +- (0,0.026983791431153047)
(10000,99.53999999999999) +- (0,0.18215378118501954)
(12000,99.23833333333333) +- (0,0.18315521650592756)
(14000,99.12142857142858) +- (0,0.11918087779779701)
(16000,99.41875) +- (0,0.07569985964319755)
(18000,99.34777777777778) +- (0,0.08109208300398789)
(20000,99.381) +- (0,0.03973034105063746)
};

    \legend{\syn{},\stat{},Metabias}
    \end{axis}
  \end{tikzpicture}
\caption{Percentage of tasks solved}
\end{subfigure}%
\begin{subfigure}{.5\linewidth}
\centering
\pgfplotsset{scaled x ticks=false}
\begin{tikzpicture}[scale=.48]
    \begin{axis}[
    xlabel=\# tasks,
    ylabel=mean learning time (seconds),
    xmin=2000,xmax=20000,
    ylabel style={yshift=-4mm},
    legend style={legend pos=north west,font=\small,style={nodes={right}}}
    ]

\addplot+[mark=square*,error bars/.cd,y fixed,y dir=both,y explicit] coordinates {
(2000,0.27958666846752167) +- (0,0.01923544033511322)
(4000,0.21560370725393296) +- (0,0.01374113139124029)
(6000,0.18799851059913636) +- (0,0.004319479825834205)
(8000,0.18046205079555513) +- (0,0.004829049649984669)
(10000,0.19346866295814516) +- (0,0.007676066713505579)
(12000,0.19868902080059053) +- (0,0.007496728862401586)
(14000,0.1989640835727964) +- (0,0.00381876384467704)
(16000,0.2021957015186548) +- (0,0.0011849363102645594)
(18000,0.2052548232369953) +- (0,0.00470199615175644)
(20000,0.2057819483566284) +- (0,0.004749483035999561)
};

\addplot+[mark=oplus*,error bars/.cd,y fixed,y dir=both,y explicit] coordinates {
(2000,0.36477195219993586) +- (0,0.02355995983066055)
(4000,0.39500249944925303) +- (0,0.031075228724662705)
(6000,0.3137922222693761) +- (0,0.012200793706958247)
(8000,0.31243912940621377) +- (0,0.010232883230100154)
(10000,0.35741282403945923) +- (0,0.014933870142622596)
(12000,0.4216530526002248) +- (0,0.022000991742844193)
(14000,0.4226379989385604) +- (0,0.016445625529271198)
(16000,0.4249608181804418) +- (0,0.01396899503654676)
(18000,0.45215285432338714) +- (0,0.012205025812162817)
(20000,0.4559692902708054) +- (0,0.0051090068005914275)
};

\addplot+[mark=diamond*,error bars/.cd,y fixed,y dir=both,y explicit] coordinates {
(2000,0.35017465243339535) +- (0,0.018028649250924456)
(4000,0.395049249112606) +- (0,0.030545457433872346)
(6000,0.3100003456274668) +- (0,0.009886867762732995)
(8000,0.30884242256879807) +- (0,0.010195964639871652)
(10000,0.3725530971670151) +- (0,0.02058435729301919)
(12000,0.39898387111822764) +- (0,0.014733386329972566)
(14000,0.4132919970648629) +- (0,0.0067453882570442285)
(16000,0.41770099714696407) +- (0,0.009174563783062507)
(18000,0.43031280159420443) +- (0,0.0097718955322369)
(20000,0.4408841947579384) +- (0,0.008935486016313655)
};
    \legend{\syn{},\stat{},Metabias}
    \end{axis}
  \end{tikzpicture}
\caption{Mean learning time}
\end{subfigure}
\caption{Robot experiment big dataset results.}
\label{fig:robot-results-big}
\end{figure}
\subsection{Experiment 2 - Lego}

Our second experiment is on building Lego structures.

\definecolor{pixel 0}{HTML}{FFFFFF}
\definecolor{pixel 1}{HTML}{FF0000} 



\subsubsection{Materials}
We consider a Lego world with dimensionality $6 \times 1$.
For simplicity we only consider $1 \times 1$ blocks of a single colour.
A training example is an atom $f(s_1,s_2)$, where $f$ is the target predicate and $s_1$ and $s_2$ are initial and final states respectively.
A state describes the Lego board as a list of integers.
The value $k$ at index $i$ denotes that there are $k$ blocks stacked at position $i$.
The goal is to learn a program to build the Lego structure from a blank Lego board (a list of zeros).
We generate training examples by generating random final states.
The learner can move along the board using the actions \emph{left} and \emph{right}; can place a Lego block using the action \emph{place\_block}; and can use the fluents \emph{at\_left} and \emph{at\_right} and their negations to determine whether it is at the leftmost or rightmost board position.
The learners use the metarules in Table \ref{tab:metarules}.

\subsubsection{Method}
The method is the same as in experiment 1, except we only use a big dataset.

\subsubsection{Results}
The results (Figure \ref{fig:lego-results-big}) show that the performance of Metabias substantially worsens after 12,000 tasks.
By 20,000 tasks Metabias can only solve around 80\% of the tasks.
By contrast, \syn{} always solves 100\% of the tasks and by 20,000 tasks can learn programs three times quicker than Metabias.
The performance of \stat{} is again similar to Metabias.
These results strongly suggest that the answers to \textbf{Q1} and \textbf{Q2} are both yes.

\begin{figure}[ht]
\centering
\begin{subfigure}{.5\linewidth}
\centering
\pgfplotsset{scaled x ticks=false}
\begin{tikzpicture}[scale=.48]
    \begin{axis}[
    xlabel=\# tasks,
    ylabel=\% tasks solved,
    xmin=2000,xmax=20000,
    ylabel style={yshift=-2mm},
    legend style={legend pos=south west,style={nodes={right}}}
    ]

\addplot+[mark=square*,error bars/.cd,y fixed,y dir=both,y explicit] coordinates {
(2000,97.86666666666667) +- (0,1.1454014337534446)
(4000,99.49166666666667) +- (0,0.2550871310835648)
(6000,99.91111111111111) +- (0,0.08069910581296653)
(8000,99.97500000000001) +- (0,0.025000000000000952)
(10000,99.95) +- (0,0.0404145188432728)
(12000,99.99166666666667) +- (0,0.004811252243246608)
(14000,99.96904761904761) +- (0,0.012598815766969945)
(16000,100.0) +- (0,0.0)
(18000,99.99074074074075) +- (0,0.00925925925926189)
(20000,100.0) +- (0,0.0)
};

\addplot+[mark=oplus*,error bars/.cd,y fixed,y dir=both,y explicit] coordinates {
(2000,97.91666666666667) +- (0,1.0536971945382503)
(4000,98.375) +- (0,0.5928392137277498)
(6000,98.37222222222222) +- (0,0.6472401046647344)
(8000,98.17916666666667) +- (0,0.6926072921296111)
(10000,94.37333333333333) +- (0,2.5583349619973137)
(12000,97.04166666666667) +- (0,2.0158630170980643)
(14000,87.57380952380953) +- (0,1.4750437174043223)
(16000,90.4375) +- (0,4.340332888539002)
(18000,90.27037037037037) +- (0,3.5876679687572346)
(20000,89.98833333333334) +- (0,7.029841945433607)

};

\addplot+[mark=diamond*,error bars/.cd,y fixed,y dir=both,y explicit] coordinates {
(2000,98.43333333333334) +- (0,0.7595685910070586)
(4000,98.02499999999999) +- (0,0.7649891066762559)
(6000,97.66666666666667) +- (0,1.4173854385739537)
(8000,96.92083333333333) +- (0,1.8609240163006235)
(10000,96.64666666666666) +- (0,1.1214474773454346)
(12000,98.20277777777778) +- (0,1.014276942481382)
(14000,88.54047619047621) +- (0,4.295766649983578)
(16000,91.80208333333333) +- (0,4.158640185739139)
(18000,91.92962962962963) +- (0,3.32202763788716)
(20000,83.90833333333335) +- (0,0.2633491556428037)
};

    \legend{\syn{},\stat{},Metabias}
    \end{axis}
  \end{tikzpicture}
\caption{Percentage of tasks solved}
\end{subfigure}%
\begin{subfigure}{.5\linewidth}
\centering
\pgfplotsset{scaled x ticks=false}
\begin{tikzpicture}[scale=.48]
    \begin{axis}[
    xlabel=\# tasks,
    ylabel=mean learning time (seconds),
    xmin=2000,xmax=20000,
    ylabel style={yshift=-4mm},
    legend style={legend pos=north west,font=\small,style={nodes={right}}}
    ]

\addplot+[mark=square*,error bars/.cd,y fixed,y dir=both,y explicit] coordinates {
(2000,3.065267617464066) +- (0,0.8447663296353122)
(4000,2.347641843934854) +- (0,0.1984833342434616)
(6000,2.35791757663091) +- (0,0.792091121817691)
(8000,1.8669702055156232) +- (0,0.1937798833661376)
(10000,1.5813738946119944) +- (0,0.014756719303078257)
(12000,1.7267290685110623) +- (0,0.12358051741962409)
(14000,1.639530023177465) +- (0,0.0316795646876658)
(16000,1.74699826695025) +- (0,0.12964593600323238)
(18000,1.6428833641652707) +- (0,0.03248735620763329)
(20000,1.6697067985574403) +- (0,0.004183196533750479)
};

\addplot+[mark=oplus*,error bars/.cd,y fixed,y dir=both,y explicit] coordinates {
(2000,3.290312786022822) +- (0,0.956247072942176)
(4000,3.1295050567388536) +- (0,0.018191134759485056)
(6000,3.4724959891637166) +- (0,0.6517439105019481)
(8000,2.9485256312489514) +- (0,0.10118694649729221)
(10000,3.5657449848254523) +- (0,0.62554141796634)
(12000,2.9227250354886056) +- (0,0.4427068357183998)
(14000,5.3719575954618906) +- (0,0.3146851438754155)
(16000,4.510564446965853) +- (0,0.7109589737117236)
(18000,5.117007283016488) +- (0,0.6120982680957212)
(20000,5.758246979709466) +- (0,0.954240902580943)
};

\addplot+[mark=diamond*,error bars/.cd,y fixed,y dir=both,y explicit] coordinates {
(2000,3.320065630237261) +- (0,0.8671591161814515)
(4000,3.303241877575715) +- (0,0.12877129994311728)
(6000,3.4409144786198933) +- (0,0.6666235645274751)
(8000,3.137581724921862) +- (0,0.3452646226717062)
(10000,3.071264066267014) +- (0,0.1335065696426777)
(12000,2.736786066492398) +- (0,0.21989826009647384)
(14000,5.109006874759992) +- (0,0.1869565652809308)
(16000,4.484368184948961) +- (0,0.7847531836755965)
(18000,4.79518017898224) +- (0,0.589397114281534)
(20000,5.942001418284575) +- (0,0.47036412633055874)
};

    \legend{\syn{},\stat{},Metabias}
    \end{axis}
  \end{tikzpicture}
\caption{Mean learning time}
\end{subfigure}
\caption{Lego experiment results.}
\label{fig:lego-results-big}
\end{figure}
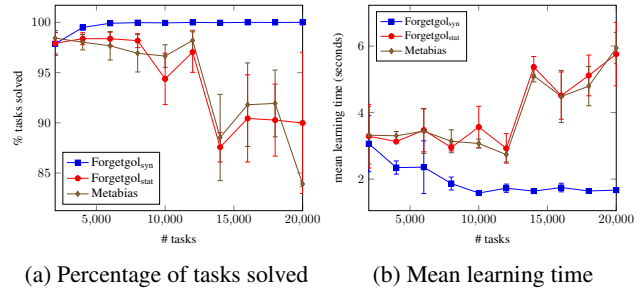

\section{Conclusions and Limitations}
We have explored the idea of \emph{forgetting}.
In this approach, a program induction system learns programs over time and is able to revise its BK (and thus its hypothesis space) by adding and removing (forgetting) learned programs.
Our theoretical results show that forgetting can substantially reduce the size of the hypothesis space (Theorem \ref{thm:forgetspace}) and the sample complexity (Theorem \ref{thm:forget:improvement}) of a learner .
We implemented our idea in \name{}, a new ILP learner.
We described two forgetting techniques: syntactical and statistical.
Our experimental results on two domains show that (1) forgetting can substantially improve learning performance (answers yes to \textbf{Q1}), and (2) remember everything approaches become overwhelmed by too much BK (answers yes to \textbf{Q2}).

\subsection{Limitations and Future Work}
There are many limitations to this work, including (1) only considering two domains, (2) only using \name{} with Metagol, and (3) not evaluating the impact of forgetting on predictive accuracy.
However, future work can easily address these minor limitations.



The main limitation of this work, and thus the main topic for future research, is our forgetting methods.
Syntactical forgetting achieves the best performance, but if every learned program is syntactically unique, then this method will forget nothing.
Future work can address this limitation by exploring semantic forgetting methods, such as those based on subsumption \cite{plotkin:thesis} or derivability \cite{crop:reduce}.
Statistical forgetting, based on expected search cost, does not perform better than Metabias (which remembers everything).
Statistical forgetting rarely forgets programs because the expected search cost is dominated by the target program size. 
Future work could improve this method by adding a \emph{frugal} dampening factor to force \name{} to forget more.
We suspect that another limitation with our forgetting methods is that they will be sensitive to concept drift \cite{DBLP:journals/ml/WidmerK96}.
For instance, if we first trained \name{} on the robot domain and then on the Lego domain, we would want it to revise its BK accordingly.
Future work could address this limitation by adding an exponential decay factor to forget programs that have not recently been used.

To conclude, we think that forgetting is an important contribution to lifelong machine learning and that our work will stimulate future research into better forgetting methods.

{
  \small

\bibliographystyle{aaai}
\bibliography{mybib}
}
\end{document}